\documentclass{article}

\usepackage{times}
\usepackage{graphicx} 
\usepackage{subcaption}

\usepackage{natbib}

\usepackage{algorithm}
\usepackage{algorithmic}

\usepackage{bm}
\usepackage{bbm}
\usepackage{amsthm}
\usepackage{amsfonts}
\usepackage{mathtools}

\usepackage{ctable}
\usepackage{multirow}

\usepackage{balance}

\usepackage{todonotes}
\usepackage{hyperref}

\theoremstyle{definition}
\newtheorem{definition}{Definition}[section]

\theoremstyle{theorem}
\newtheorem{theorem}{Theorem}[section]

\theoremstyle{lemma}
\newtheorem{lemma}{Lemma}[section]

\usepackage[accepted]{icml2016}

\icmltitlerunning{Learning Purposeful Behaviour in the Absence of Rewards}

\begin{document} 

\twocolumn[
\icmltitle{Learning Purposeful Behaviour in the Absence of Rewards}

\icmlauthor{Marlos C. Machado}{machado@ualberta.ca}
\icmlauthor{Michael Bowling}{mbowling@ualberta.ca}
\icmladdress{Department of Computing Science, University of Alberta, Canada}

\icmlkeywords{Reinforcement Learning, Option Discovery, Exploration}

\vskip 0.3in
]

\begin{abstract} 
Artificial intelligence is commonly defined as the ability to achieve goals in the world. In the reinforcement learning framework, goals are encoded as reward functions that guide agent behaviour, and the sum of observed rewards provide a notion of progress. However, some domains have no such reward signal, or have a reward signal so sparse as to appear absent.  Without reward feedback, agent behaviour is typically random, often dithering aimlessly and lacking intentionality.  In this paper we present an algorithm capable of learning purposeful behaviour in the absence of rewards.  The algorithm proceeds by constructing temporally extended actions (options), through the identification of purposes that are ``just out of reach'' of the agent’s current behaviour.  These purposes establish intrinsic goals for the agent to learn, ultimately resulting in a suite of behaviours that encourage the agent to visit different parts of the state space.  Moreover, the approach is particularly suited for settings where rewards are very sparse, and such behaviours can help in the exploration of the environment until reward is observed.
\end{abstract} 


\section{Introduction}
\label{intro}

Reinforcement learning (RL) has been successful in generating agents capable of intelligent behaviour in initially unknown environments; with accomplishments such as surpassing human-level performance in Backgammon \cite{Tesauro95}, helicopter flight \cite{Ng04}, and general competency in dozens of Atari 2600 games \cite{Mnih15}. Such successes are achieved by algorithms that maximize the expected cumulative sum of rewards, which can be seen as a measure of progress towards the desired goal. The goal is sometimes easily defined through rewards, such as $\pm1$ signals encoding win/loss in games or $-1$ signals informing the agent something undesirable occurred (\emph{e.g.,} a robot bumping into a wall). 

We are interested in the setting where the reward signal is uninformative or even absent. In the uninformative setting, the paucity of reward usually leads to dithering typical of $\epsilon$-greedy exploration approaches. This effect is particularly pronounced when the agent operates at a fine time scale, as is common of video game platforms \citep{Bellemare13}. In the complete absence of reward, it is unclear what intelligent behaviour should even constitute. Intrinsic motivation-based approaches \citep{Singh04,Oudeyer07,Barto13} offer a solution in the form of an intrinsic reward signal, and some authors have proposed agents undergoing developmental periods in which they are not concerned with maximizing extrinsic reward but in acquiring reusable options autonomously learned from intrinsic rewards \cite{Singh04}. However, here we instead consider the notion of a purposeful agent: one that can commit to a behaviour for an extended period of time.

To construct purposeful agents, we appeal to the options framework \citep{Sutton99}. Options extend the mathematical framework of reinforcement learning and Markov decision processes (MDPs) to allow agents to take temporally extended actions to accomplish subgoals. While this extension is an extremely powerful idea for allowing reasoning at different levels of abstraction, automatically discovering options (\emph{e.g.}, by identifying subgoals) is an open-problem in the literature. Generally, options are designed by practitioners who identify meaningful subgoals that can be used as stepping-stones to solve a complex task. Besides using options to divide a task in to subgoals, we advocate one can also use options to add decisiveness to agents in an environment in which rewards are not available, and that this is a better choice than aimless exploration.

In this paper we introduce an algorithm capable of learning purposeful behaviour in the absence of rewards. Our approach discovers options by identifying purposes that are achievable by the agent. These purposes are turned into intrinsic subgoals through to an intrinsic reward function, resulting in a suite of behaviours that encourage the agent to visit different parts of the state space. These options are particularly useful in the absence of rewards. As an example, when the agent observes a change in the environment through its feature representation, it tries to learn a policy capable of reproducing that change. Also, when such an option is added to the agent's action set, the agent now can move farther in the state-space, with some events that were rare now becoming frequent and events that were ``impossible'' now becoming ``just'' infrequent.

In this early paper we introduce the main ideas that underly our algorithm, including concepts such as ``purpose''. We also provide an algorithm for option discovery with linear function approximation, in contrast to most approaches for option discovery that rely on tabular representations. We show that in any finite MDP our learned options are guaranteed to have at least one state which will cause option termination.  Finally, we apply our approach to a simple domain, showing it can reach states further from the starting state, thus exhibiting intentionality in its behaviour.

\section{Background}

In this section we introduce the reinforcement learning~(RL) problem setting and the options framework. We also discuss the problem of option discovery and some approaches that try to address it. As a convention, we indicate random variables by capital letters (\emph{e.g.}, $S_t$, $R_t$), vectors by bold letters (\emph{e.g.}, $\boldsymbol{\theta}$), functions by lowercase letters (\emph{e.g.}, $v$), and sets by calligraphic font (\emph{e.g.}, $\mathcal{S}$, $\mathcal{A}$).

\subsection{Reinforcement Learning and Options}

In the RL framework \citep{Sutton98,Szepesvari10} an agent aims at maximizing some notion of cumulative reward by taking actions in an environment; these actions may affect the next state the agent will be as well as all subsequent rewards it will experience. It is generally assumed the tasks of interest satisfy the Markov property, being called Markov decision processes (MDPs). An MDP is formally defined as a 5-tuple $\langle\mathcal{S}, \mathcal{A}, r, p, \gamma\rangle$. At time $t$ the agent is in the state $s_t \in \mathcal{S}$ where it takes an action $a_t \in \mathcal{A}$ that leads to the next state $s_{t+1} \in \mathcal{S}$ according to the transition probability kernel $p(s'|s, a)$, which encodes $Pr(S_{t+1} = s' | S_{t} = s, A_{t} = a)$. The agent also observes a reward $R_{t+1} \sim r(s, a, s')$. The agent's goal is to learn a policy $\pi : \mathcal{S} \times \mathcal{A} \rightarrow [0,1]$ that maximizes the expected discounted return $G_t \doteq \mathbb{E}_\pi \big[\sum_{k=0}^{\infty} \gamma^k R_{t+k+1} | s_t\big]$, where $\gamma \in [0, 1)$ is known as the discount factor. We are interested in settings where the reward signal $R_t$ is uniform across the environment.

When learning to maximize $G_t$ it is common to learn an \emph{action-value function} defined as $q_\pi(s,a) \doteq \mathbb{E}_\pi[G_t| S_t = s, A_t = a]$. 
However, in large problems it may be infeasible to learn $q_\pi$ exactly for each state-action pair. To tackle this issue agents often learn an approximate value function: $q_\pi(s,a) \approx q_\pi(s,a; \boldsymbol{\theta})$. A common approach is to approximate these values using linear function approximation where $q_\pi(s, a; \boldsymbol{\theta}) \doteq \boldsymbol{\theta}^\top\phi(s,a)$, in which $\boldsymbol{\theta}$ denotes the vector of weights and $\phi(s,a)$ denotes a static feature representation of the state $s$ when taking action $a$. This can also be done through non-linear function approximation methods such as neural networks (\emph{e.g.}, \citeauthor{Tesauro95},~\citeyear{Tesauro95}, \citeauthor{Mnih15},~\citeyear{Mnih15}). Note that generally $\boldsymbol{\theta}$ has much smaller number of parameters than the number of states in $\mathcal{S}$. 

The standard RL framework is focused on MDPs, in which actions last a single time step. Nevertheless, it is convenient to have agents encoding higher levels of abstraction, which also facilitate the learning process if properly defined \citep{Dietterich98, Sutton99}. \citeauthor{Sutton99} extended the RL framework by introducing temporally extended actions called \emph{options}. Intuitively, options are higher-level actions that are extended through several time steps. Formally, an option $o$ is defined as 3-tuple $o = \langle \mathcal{I}, \varpi, \mathcal{T} \rangle$ where $\mathcal{I} \in \mathcal{S}$ denotes the initiation set, $\varpi : \mathcal{A} \times \mathcal{S} \rightarrow [0,1]$ denotes the option's policy, and $\mathcal{T} \in \mathcal{S}$ denotes the termination set. After initiated, actions are selected according to $\varpi$ until the agent reaches a state in $\mathcal{T}$. Originally, \citeauthor{Sutton99} defined a function $\beta : \mathcal{S} \rightarrow [0,1]$ to encode the probability of an option terminating at a given state. In this paper we define $\beta$ to be the characteristic function of the set $\mathcal{T}$: $\beta(s) = 1$ for all $s \in \mathcal{T}$ and $\beta(s) = 0$ for all $s \notin \mathcal{T}$, hence we overload the notation. Options generalize MDPs to semi-Markov decision processes (SMDPs) in which actions take variable amounts of time. Options are also useful when addressing the problem of exploration because they can move agents farther in the state-space.

\subsection{Option Discovery}

The potential of options to dramatically affect learning by improving exploration is well known (\emph{e.g.}, \citeauthor{McGovern98}, \citeyear{McGovern98}; \citeauthor{McGovern01}, \citeyear{McGovern01}; \citeauthor{Kulkarni16}, \citeyear{Kulkarni16}). Nevertheless, most works that benefit from options do not discover them, but assume they are provided or that there is a hardwired notion of interestingness (reward) that can be used to discover options \emph{e.g.}, salient events \cite{Singh04}.

The works that investigate how to discover options can be clustered in three different categories. The most common approach is to try to identify subgoal states through some heuristic such as visitation frequency~\cite{McGovern01}, graph-related metrics such as betweenness~\cite{Simsek08}, or graph partitioning metrics~\cite{Menache02, Mannor04, Simsek05}. Some authors have also tackled the problem of option discovery by trying to identify common subpolicies~\cite{Thrun94, Pickett02}, while others proposed methods based on the frequency of the change of state variables~\cite{Hengst02}.

The works on option discovery generally operate in a tabular setting where you can have states uniquely defined. Consequently, metrics such as frequency of visitation and transition graphs can be used for option discovery. Automatically discovering options in large state-spaces where function approximation is required is still a challenge. Our work presents an approach for option discovery in settings with linear function approximation, which has a much larger applicability. Few works tackled option discovery with function approximation. Those that did generally simplified the problem with additional assumptions such as knowledge of subgoal states~\cite{Konidaris09} or that one can control the interface between regions of the MDP~\cite{Hengst02}.

The proposals of not depending on a reward signal to discover meaningful behaviour~\cite{Simsek04}, and of looking at the different rates of changes in the agent's feature representation~\cite{Hengst02} are related to our work. This relationship will be clearer in the next section.

\section{Option Discovery for Purposeful Agents}


Approaches based on intrinsic motivation and novelty are some of the ways to circumvent the absence of rewards in the environment. \citeauthor{Schmidhuber10}~(\citeyear{Schmidhuber10}) summarizes several works based on intrinsic motivation, which he defines as algorithms that maximize a reward signal derived from the agent's learning progress. \citeauthor{Lehman11}~(\citeyear{Lehman11}) have advocated that agents should drop feedback they receive from the environment even in more traditional settings such as search, maximizing novelty instead. Both ideas are related to our work. We discover options based on novelty assuming no extrinsic rewards are available \cite{Lehman11}. These options are based on a very loose notion of a model of the environment, aiming at learning how to change principal components of a compressed environment representation \cite{Schmidhuber10}.

Our algorithm is based on four different concepts, namely: (i) storing the changes seen between two different time steps, (ii) clustering correlated changes in order to extract a purpose, (iii) learning policies capable of reproducing desired purposes, and (iv) transforming these policies into options that can be used to move the agent farther in the state space. After these steps a new set of options giving different purposes to the agent is discovered. These options guide the agent to different parts of the state space, which may lead to identifying new purposes. When such steps are used iteratively we create a self-reinforcing loop. We discuss each concept individually while introducing the algorithm. The algorithm we introduce uses a binary feature representation, but it is extensible to more general settings.

The agent initially follows some default policy (possible random) for a given number of time steps, using all actions available on its action set. While following such policy, at every time step the agent stores in a dataset $\mathcal{D}$ the difference between the feature representation of its current observation $\phi(s_t)$ and the representation of its previous observation $\phi(s_{t-1})$, \emph{i.e.}: $\mathcal{D} \leftarrow \mathcal{D} \cup \{\big(\phi(s_t) - \phi(s_{t-1})\big)\}$. It is important to stress that while storing changes one can easily see those that stand out, such as features that rarely change. Storing the features is less informative than the current transition because it is harder for the agent to identify when a feature really changed. Moreover, storing differences allow us to clearly identify different directions in the change, something that is useful in the next steps.

After a pre-defined number of time steps the agent stops storing transitions to identify purposes in the observed behaviour. It clusters together features that change together, avoiding correlated changes to generate the same purpose. Formally, such step consists in reducing the dimensionality of $\mathcal{D}$ through singular value decomposition~(SVD): $ \mathcal{D} = U \Sigma V^*$. The SVD generates a lower rank representation of the transitions stored in $\mathcal{D}$. Such low rank representation consists of eigenvalues and eigenvectors. The eigenvectors encode the principal components of $\mathcal{D}$ while the eigenvalues weight them according to how much that eigenvector is important to reconstruct $\mathcal{D}$. Each eigenvector can be seen as encoding a different purpose for the agent because all features that are somehow correlated are collapsed to a single eigenvector, explaining a direction of variation of the observed transitions. We call the eigenvectors obtained from the dataset of transitions \emph{eigenpurposes}.

\theoremstyle{definition}
\begin{definition}[Eigenpurpose]
Given a matrix $\mathcal{D}$ of transitions where each row encodes the difference between two consecutive observations, \emph{i.e.} $\phi(s_t) - \phi(s_{t-1})$, and having $V_i$ denoting the $i$-th row of matrix $V$; each eigenvector $(V^*)_i$ obtained from the singular value decomposition traditionally defined as $\mathcal{D} = U \Sigma V^*$ is called an \emph{eigenpurpose}.
\end{definition}

The following example provides an intuition about the importance of eigenpurposes. Imagine that an agent, by chance, leaves a building. By doing it the value of several of its features change (\emph{e.g.} lighting, temperature, soil). When we collapse all these changes with the SVD, instead of having a feature encoding ``temperature increase'', other encoding ``change of lighting", and so on, we have only an eigenpurpose encoding ``outside the building''.

Once eigenpurposes have been identified, the agent learns policies capable of maximizing their occurrence. In order to learn a policy that maximizes an eigenpurpose we need to define a proper reward function. Such reward $r_{i,t}$ is defined as the similarity between the observed transition and the eigenpurpose of interest $\boldsymbol{e_i}$, $r_{i, t} = \boldsymbol{e_i}^\top \big(\phi(s_{t}) - \phi(s_{t-1}) \big)$. Because SVD does not encode signs, we learn how to maximize eigenpurposes in both possible directions, \emph{i.e.} the agent also learns a second policy that maximizes $-r_{i, t}$. These policies are called \emph{eigenbehaviours}.
\theoremstyle{definition}
\begin{definition}[Eigenbehaviour]
A policy $\pi: \mathcal{S} \rightarrow \mathcal{A}$ is called an \emph{eigenbehaviour} if it is the optimal policy that maximizes the occurrence of an \emph{eigenpurpose} (Definition~3.1) in the original MDP augumented by the action $\bot$ that terminates the option; \emph{i.e.} $\pi(s) = \arg\max_a \max_\pi q_\pi(s,a)$ in the MDP $\langle \mathcal{S}, \mathcal{A}~\cup~\{\bot\}, (V^\top)_j \big(\phi(s_t) - \phi(s_{t-1})\big), p, \gamma \rangle$.
\end{definition}
The algorithm used to learn eigenbehaviours is not pre-defined, nor the order in which they are learned. 

We can naturally construct an option from the learned eigenbehaviour. To do so, we need to define the set of states for which the eigenbehaviour is effective (initiation set) and its termination condition. We define the initiation set of an option as the set of states $s$ in which, after learning, $q_\pi(s,a) > 0$ for at least one action $a \in \mathcal{A}$. Intuitively this corresponds to every state in which the policy can still make progress towards the eigenpurpose. The set of terminal states for this option is the complement of the initiation set, \emph{i.e.}~$\mathcal{S} \setminus \mathcal{I}$. Once such options are discovered we can add them to the agents action set, which allows the agent to repeat the described process with a policy that uses the discovered options to collect new data.

Notice that eigenvalues loosely encode how frequent each eigenpurpose was observed. Therefore, the eigenbehaviours corresponding to the lower eigenvalues encode purposes observed less frequently. Because of that, once an option for a ``rare'' purpose is discovered, this purpose will no longer be ``rare'' since a single decision (taking the option) is now capable of reproducing this rare event. We may also increase the likelihood of observing other unlikely ``purposes'' since a single action now moves the agent much farther in the state-space. This can help agents to explore environments in which rewards are very sparse, guiding the agent until a reward signal is observed.

The described algorithm is formally presented in Algorithm~1. An additional detail not discussed yet is that one can decide to learn how to maximize only a subset of the discovered eigenpurposes. In this work we did not evaluate the impact of different approaches. Here we propose a simple eigenvalue threshold $\kappa$ that determines which eigenpurposes are interesting. We select all eigenpurposes that have a correspondent eigenvalue greater than $\kappa$, which we interpret as discarding noise.

\renewcommand{\algorithmicrequire}{\textbf{Input:}}
\renewcommand{\algorithmicensure}{\textbf{Output:}}

\begin{algorithm}[t]
\caption{Purposeful Option Discovery (POD)}
\label{alg1}
\begin{algorithmic}[1]
    \REQUIRE \ \ $\mathcal{A}$
    \hspace{1.21cm} \COMMENT{Action set}\\
     \ \ $\kappa$ 
    \hspace{1.83cm} \COMMENT{Noise threshold}\\
     \ \ $n_I > 0$
    \hspace{1.01cm} \COMMENT{Number of iterations}\\
     \ \  $n_R > 0$
    \hspace{0.95cm} \COMMENT{Num. of random steps per iteration}
    
    \ENSURE $\Omega$
    \hspace{1.15cm} \COMMENT{Option set}
    
    \vspace{0.1cm}
    
    \STATE $\Omega \leftarrow \emptyset$

    \FOR {$i \leftarrow 1$ {\bf to} $n_I$}

    \STATE $\mathcal{D} \leftarrow \emptyset$

    \FOR {$j \leftarrow 1$ {\bf to} $n_R$}

    \STATE {Observe $\phi(s)$}
    
    \STATE Take an action $a \in \mathcal{A}$ or an option $o \in \Omega$ 

    \IF {option $o$ was taken}

    \WHILE {$s \notin \mathcal{T}_o$ \textbf{and} $j < n_R$}

    \STATE Take an action $a$ following $\varpi_o$

    \STATE {Observe $\phi(s')$}

    \STATE $\mathcal{D} \leftarrow \mathcal{D} \ \cup \ \big(\phi(s') - \phi(s) \big)$

    \STATE {Observe $\phi(s)$}

    \STATE $j \leftarrow j + 1$

    \ENDWHILE

    \ELSE

	\STATE Observe $\phi(s')$    

	\STATE $\mathcal{D} \leftarrow \mathcal{D} \ \cup \ \big(\phi(s') - \phi(s) \big)$

    \ENDIF

    \ENDFOR
    
    \vspace{0.1cm}
    
    \STATE $U, \Sigma, V \leftarrow $ \textsc{SVD}$(\mathcal{D}$)

    \vspace{0.1cm}

    \STATE {\bf for all} $\mathbf{j}$ such that $\mathbf{\Sigma_j > \kappa}$
    
    \STATE \hspace{0.25cm} Learn policy $\pi_j$ that max. $\big(V^\top \big)_j \big(\phi(s') - \phi(s) \big)$

    \STATE \hspace{0.25cm} Learn policy $\pi_k$ that max. $\big(-V^\top \big)_j \big(\phi(s') - \phi(s) \big)$

    \STATE \hspace{0.25cm} {$\mathcal{I}_j \leftarrow \{s | s \in \mathcal{S}, \exists a \in \mathcal{A}: q_{\pi_j}(s,a) > 0\}$}

    \STATE \hspace{0.25cm} {$\mathcal{I}_k \leftarrow \{s | s \in \mathcal{S}, \exists a \in \mathcal{A}: q_{\pi_k}(s,a) > 0\}$}
    
    \STATE \hspace{0.25cm} {$\Omega \leftarrow \Omega \cup \langle \mathcal{I}_j, \pi_j, \mathcal{S} \setminus \mathcal{I}_j \rangle \cup \langle \mathcal{I}_k, \pi_k, \mathcal{S} \setminus \mathcal{I}_k \rangle$}

    \STATE {\bf end for all}

    \ENDFOR
\end{algorithmic}
\end{algorithm}

Our constructed options represent a ``purpose'', which can be thought of as reaching states in the option’s termination set, which we can show is guaranteed to be non-empty.

Finally, it is important to guarantee that there is at least one state that satisfies the discovered purposes, or in other words, that the termination set of an option is not empty.

\begin{theorem}[Option's Termination]
Consider an option $o = \langle\mathcal{I}_o, \pi_o, \mathcal{T}_o\rangle$ discovered with Algorithm~1 where $\gamma~<~1$. Then $\mathcal{T}_o$ is nonempty.
\end{theorem}
\begin{proof}[Proof intuition]
Consider the state with the largest potential value.  From this state the agent must terminate either due to the discount factor or due to the termination action in a state with lesser or equal potential value.  The cumulative reward received is the difference in potential and so the expected value of the state must be non-positive. The complete proof is in the Appendix.
\end{proof}


\section{Experimental Evaluation}

We performed an empirical validation in a hand-crafted domain which allows us to clearly illustrate our algorithm's features, namely:
\begin{itemize} \setlength\itemsep{0cm}
\item At each new iteration of the algorithm the discovered options become increasingly more complex.
\item More complex options are able to move the agent farther away in the state space.
\item As the agent moves farther away with newly discovered options, it observes new eigenpurposes, discovering more options, what creates a self-reinforcing loop.
\end{itemize}

We evaluated our algorithm in a random walk in a toy domain consisting of moving around a ring of length 4096 with deterministic transitions. The agent starts at the $x$ coordinate 0 and at every time step it chooses between going right or going left. We use linear function approximation with the two's complement binary encoding as representation (12 bits long). When the agent goes left on the state 0 it goes to the state $-1$ ($0000\ 0000\ 0000_2 \rightarrow 1111\ 1111\ 1111_2$). Similarly, going right in state $2047$ transitions to $-2048$. There are no rewards in this environment.

All the evaluations were made in the same setting. The selected ``noise'' parameter $\kappa$ was set to $1$ and the environment, when learning the eigenbehaviours, had a discount factor $\gamma = 0.99$. The policies were obtained through value iteration with $100$ iterations. Each round of our algorithm consisted of $1,000$ time steps in which the agent collected transitions to discover eigenpurposes. We ran six rounds, with round zero having only primitive actions available. The agent's default policy was to select uniformly random among the currently available actions, or options.

We first evaluated our results in agents with full observability, in which the agent perceives states as described above (Figure \ref{fig:full}, and Table~1). By looking at the average length of the discovered options we see that options become increasingly complex with each iteration. This added complexity allows the agent to move to farther states, as evidenced by the increasing distance between farthest point and the agent's starting state at that iteration (Max Dist. from Start). The improvement is particularly clear when comparing to a sample random walk on primitive actions (Figure \ref{fig:full}). Note that, despite options constructed in later iterations still use only primitive actions, they present more complex behaviours. This is different than typical option discovery methods, which construct hierarchies of options. 

We also evaluated a setting in which the agent had partial observability. In this setting the agent does not observe the three least significant bits encoding the state. This collapses several states together and makes it much harder for the agent to observe progress. Interestingly, the same behavioural pattern as in the full observability experiment emerges. Agents still come up with options of the type ``flip the $i$-th bit'' once they discover these purposes. The unique difference is that fewer options are discovered at each iteration, due to fewer observable eigenpurposes.

\begin{table*}
\begin{center}
\scriptsize{
\caption{Characteristics of discovered options, per iteration, when compared to a random walk in a ring. Each number is the average of 30 runs and standard deviations are reported between parentheses. For details about the selected parameters \emph{c.f.} text.}
\begin{tabular}{ c | l | r l | r l | r l | r l | r l | r l}
  \specialrule{.1em}{.05em}{.05em}
  Observ. &Metric & \multicolumn{2}{c|}{Iter. 0} & \multicolumn{2}{c|}{Iter. 1} & \multicolumn{2}{c|}{Iter. 2} & \multicolumn{2}{c|}{Iter. 3} & \multicolumn{2}{c|}{Iter. 4} & \multicolumn{2}{c}{Iter. 5}\\
  \specialrule{.1em}{.05em}{.05em}
  \parbox[t]{2mm}{\multirow{3}{*}{\rotatebox[origin=c]{90}{Full}}}
   &Num. Options Discov.  & -    & (-)    & 5.9   & (5.5)   & 7.7   & (10.3)  & 8.5   & (8.8)   & 9.2   & (12.6)  & 9.5   & (11.4)  \\
   &Avg. Opt. Length      & -    & (-)    & 12.1  & (1.1)   & 19.2  & (2.1)   & 21.6  & (2.0)   & 25.5  & (2.0)   & 27.8  & (1.7)   \\
   &Max Dist. from Start  & 29.3 & (18.2) & 168.7 & (222.5) & 240.1 & (287.3) & 269.9 & (311.5) & 287.1 & (436.9) & 298.9 & (320.8) \\
  \hline
  \parbox[t]{2mm}{\multirow{3}{*}{\rotatebox[origin=c]{90}{Partial}}}
   &Num. Options Discov.  & -    & (-)    & 3.5   & (20.9)  & 5.2   & (14.1)  & 6.2   & (19.4)  & 6.6   & (30.2)  & 6.8   & (21.6)  \\
   &Avg. Opt. Length      & -    & (-)    & 20.4  & (1.8)   & 30.5  & (1.4)   & 33.5  & (2.1)   & 35.9  & (1.7)   & 37.7  & (1.7)   \\
   &Max Dist. from Start  & 29.3 & (18.2) & 212.8 & (326.9) & 314.9 & (314.3) & 301.8 & (434.3) & 352.4 & (464.1) & 301.1 & (360.0) \\
  \specialrule{.1em}{.05em}{.05em}
\end{tabular}
}
\end{center}
\end{table*}

\begin{figure}
    \centering
    \begin{subfigure}[b]{0.23\textwidth}
        \includegraphics[width=\textwidth]{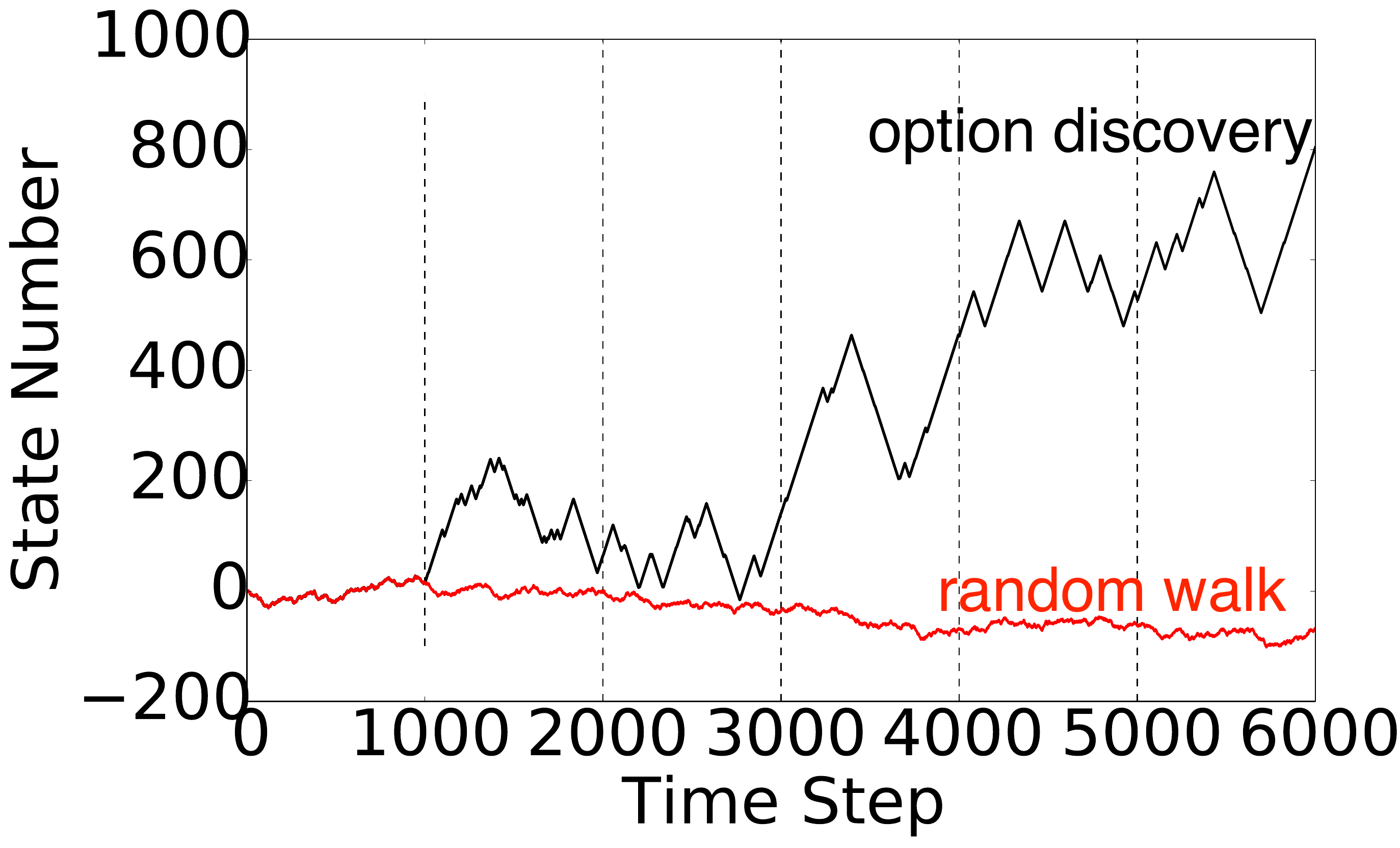}
        \caption{Full observability}
        \label{fig:full}
    \end{subfigure}
    ~ 
    \begin{subfigure}[b]{0.23\textwidth}
        \includegraphics[width=\textwidth]{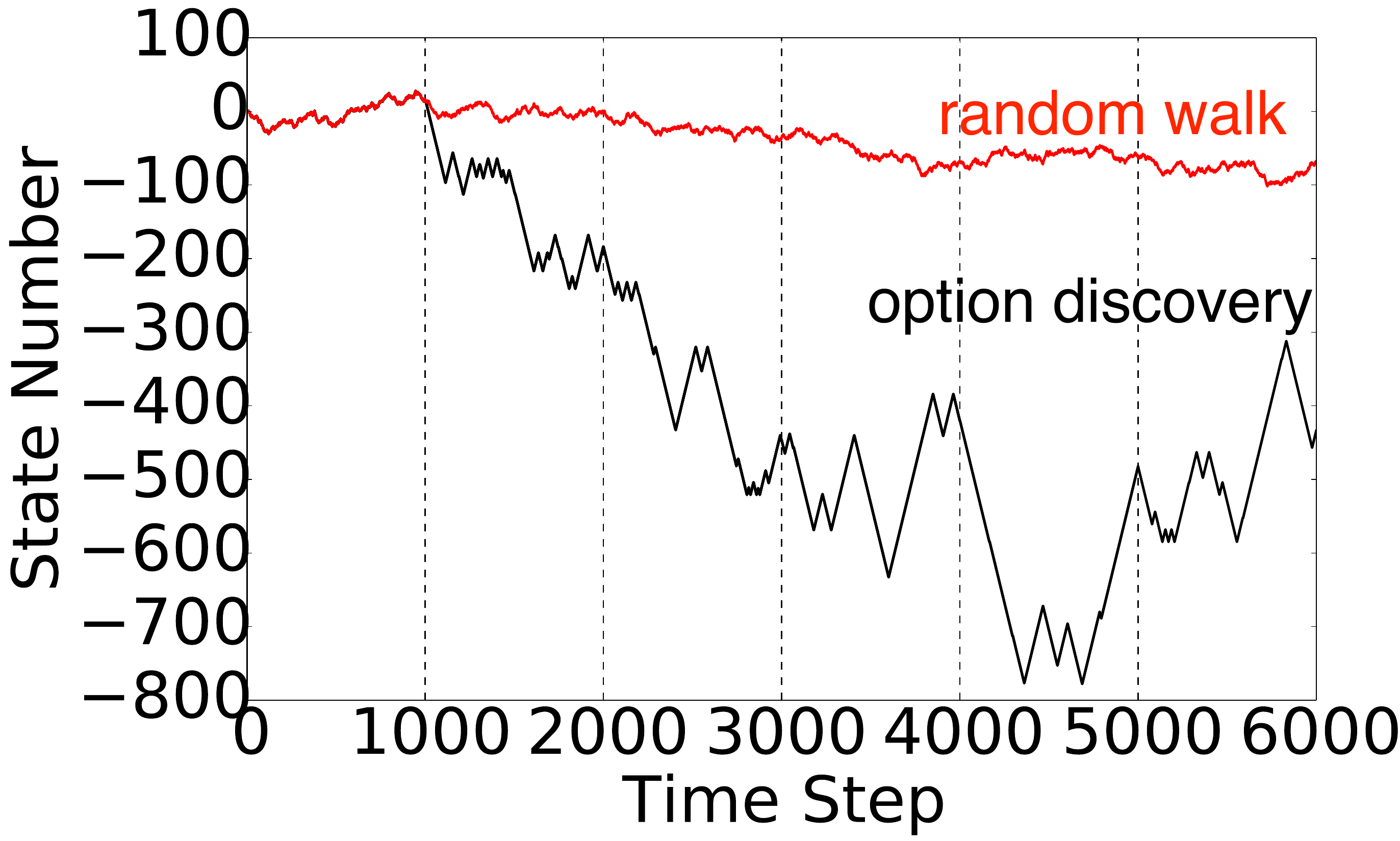}
        \caption{Partial observability}
        \label{fig:partial}
    \end{subfigure}
    \caption{Sample random walk using primitive actions and a random walk using the discovered options. Dashed vertical lines represent iteration boundaries.}\label{fig:options}
\end{figure}

\section{Conclusion}

In this paper we introduced a new algorithm capable of discovering options without any feedback in the form of rewards from the environment. Our algorithm discovers options that reproduce purposes extracted from the states seen by an agent acting randomly. We presented experimental results showing how the discovered options greatly improve exploration by introducing decisiveness on the agents, avoiding the traditional aimless dithering. We also showed evidences that our approach may work well with partial observability.

As future work, we plan to investigate how this algorithm behaves in more complex environments, such as the Arcade Learning Environment~\cite{Bellemare13}. We can evaluate our algorithm on these large domains because it is amenable to function approximation, differently from most other approaches for option discovery. Naturally, we then have to be able to learn a policy, using the discovered options, to maximize the discounted sum of rewards. Some of our preliminary results using interrupting options do seem promising. However, when applying this algorithm to larger domains we face a challenge not discussed here: the exploding number of eigenpurposes (and consequently discovered options), indicating that proper sampling techniques must be further evaluated. Finally, it is important to have coevolving action and representation abstractions: higher levels of action abstraction should drive the agent to improve its representation of the world and once the agent has a better representation of the world, better action abstractions should become available. This is a topic that is not commonly investigated but that needs to be addressed in the future, maybe our algorithm can be the first step towards that direction.

\vspace{-0.01cm}

\section*{Acknowledgements}
The authors thank Richard S. Sutton and Marc G. Bellemare for insightful discussions that helped improve this work, and Csaba Szepesv\'ari for point out the Neumann series we used in our proof. This work was supported by grants from Alberta Innovates – Technology Futures and the Alberta Innovates Centre for Machine Learning (AICML). 

\balance

\bibliography{references}
\bibliographystyle{icml2016}

\onecolumn
\section*{Appendix}

\begin{lemma}
 Suppose $(I + A)$ is a non-singular matrix, with $||A|| \leq 1$. We have:
$$||(I + A)^{-1}|| \leq \frac{1}{1 - ||A||}.$$
\end{lemma}

\begin{proof}\footnote{Our proof follows closely the proof of Parnell in lecture notes available at \url{http://www-solar.mcs.st-and.ac.uk/~clare/Lectures/num-analysis.html}.}
\begin{align*}
(I + A)(I + A)^{-1} &= I\\
I(I + A)^{-1} + A(I + A)^{-1} &= I\\
(I + A)^{-1} &= I - A(I + A)^{-1}\\
||(I + A)^{-1}|| &= ||I - A(I+A)^{-1}||\\
                 &\leq ||I|| + ||A (I + A)^{-1}|| && \text{because} \ \ ||A + B|| \leq ||A|| + ||B||\\
                 &\leq 1 + ||A||||(I + A)^{-1}|| && \text{because} \ \ ||A B|| \leq ||A|| \cdot ||B||\\
||(I + A)^{-1}|| - ||A||||(I + A)^{-1}||&\leq 1\\
(1-||A||) ||(I + A)^{-1}|| &\leq 1\\
||(I + A)^{-1}|| &\leq \frac{1}{1 - ||A||} && \text{if} \ \ ||A|| \leq 1.
\end{align*}
\end{proof}

\begin{lemma}
The induced infinity norm of $(I - \gamma T)^{-1}T$ is bounded by
$$||(I - \gamma T)^{-1}T||_\infty \leq \frac{1}{(1 - \gamma)}.$$
\end{lemma}

\begin{proof}
\begin{align*}
||(I - \gamma T)^{-1}T||_\infty & \leq ||(I - \gamma T)^{-1}||_\infty||T||_\infty  && \text{because} \ \  ||AB||_{\infty} \leq ||A||_{\infty} \cdot ||B||_{\infty}\\
||(I - \gamma T)^{-1}T||_\infty & \leq \frac{1}{1 - ||-\gamma T||_\infty} ||T||_\infty && \text{Lemma 3.1} \\
||(I - \gamma T)^{-1}T||_\infty & \leq \frac{1}{1 - \gamma ||T||_\infty} ||T||_\infty && \text{because}  \ \ ||\lambda B|| = |\lambda| ||B|| \\
||(I - \gamma T)^{-1}T||_\infty & \leq \frac{1}{(1 - \gamma)}
\end{align*}
\end{proof}

\begin{theorem}[Option's Termination]
Consider an option $o = \langle\mathcal{I}_o, \pi_o, \mathcal{T}_o\rangle$ discovered with Algorithm~1 where $\gamma~<~1$. Then $\mathcal{T}_o$ is nonempty.
\end{theorem}

\begin{proof}
We can write the Bellman equation in the matrix form: $\bm{v} = R + \gamma T \bm{v}$, where $v$ is a \emph{finite} column vector with one entry per state encoding its value function. From Algorithm~1 we have $R = T\bm{w} - \bm{w}$ with $\bm{w} = \phi(s)^\top \bm{e}$, where $\bm{e}$ denotes the eigenpurpose of interest. Therefore:

\begin{align*}
\bm{v} &= T\bm{w} - \bm{w} + \gamma T \bm{v}\\
\bm{v} + \bm{w} &= T\bm{w} + \gamma T \bm{v}\\
                &= T\bm{w} + \gamma T \bm{v} + \gamma T \bm{w} - \gamma T \bm{w}\\
                &= (1- \gamma) T \bm{w} + \gamma T (\bm{v} + \bm{w})\\
\bm{v} + \bm{w} - \gamma T (\bm{v} + \bm{w}) &= (1 - \gamma) T \bm{w}\\
(I - \gamma T) (\bm{v} + \bm{w}) &= (1 - \gamma) T \bm{w}\\
 \bm{v} + \bm{w} &= (1 - \gamma) (I - \gamma T)^{-1} T \bm{w} && (I - \gamma T)^{-1} \ \text{is guaranteed to be nonsigular because}\\
 & && ||T|| \leq 1 \text{, where } \ ||T|| = \sup_{\mathbf{v}:||\mathbf{v}||_\infty = 1} ||T\mathbf{v}||_\infty \text{. By }\\
 & && \text{Neumann series we have } (I - \gamma T)^{-1} = \sum_{n=0}^\infty \gamma^nT^n\\
 ||\bm{v} + \bm{w}||_\infty &= (1 - \gamma)||(I - \gamma T)^{-1} T \bm{w}||_\infty && \text{using the induced norm}\\
 ||\bm{v} + \bm{w}||_\infty &\le (1 - \gamma)||(I - \gamma T)^{-1} T||_\infty ||\bm{w}||_\infty && \text{because $||A\bm{x}||\leq||A||\cdot||\bm{x}||$}\\
 ||\bm{v} + \bm{w}||_\infty &\le (1 - \gamma) \frac{1}{(1-\gamma)} ||\bm{w}||_\infty && \text{Lemma~3.2}\\
 ||\bm{v} + \bm{w}||_\infty &\le ||\bm{w}||_\infty\\
\end{align*}
We can shift $\bm{w}$ by any finite constant without changing the reward, \emph{i.e.} $T\bm{w} - \bm{w} = T(\bm{w} + \bm{\delta}) - (\bm{w} + \bm{\delta})$ because $T\bm{1}\delta = \bm{1}\delta$ since $\sum_j T_{i,j} = 1$. Therefore, we can assume $\bm{w} \ge \bm{0}$. Let $s^* = \arg\max_s \bm{w}_{s^*}$, so that $\bm{w}_{s^*} = ||\bm{w}||_\infty$. Clearly $\bm{v}_{s^*} \le 0$, otherwise $||\bm{v} + \bm{w}||_\infty \ge |\bm{v}_{s^*} + \bm{w}_{s^*}| = \bm{v}_{s^*} + \bm{w}_{s^*} > \bm{w}_{s^*} = ||\bm{w}||_\infty$, arriving at a contradiction.
\end{proof}

\end{document}